\documentclass[12pt]{amsart}

\usepackage[notref,notcite]{showkeys}
\usepackage{fullpage}
\usepackage{amsfonts,amssymb}
\usepackage{enumerate}
\usepackage{verbatim}
\usepackage{graphicx}
\usepackage{tikz-cd}
\usetikzlibrary{matrix}
\usepackage{comment}

\usepackage{float}

\newcommand{\gotc}{{\mathfrak c}}

\newcommand{\R}{\mathbb{R}}
\newcommand{\Sp}{\mathbb{S}}

\newcommand{\Z}{\mathbb{Z}}
\newcommand{\NN}{\mathcal {N}}

\newcommand{\calH}{{\mathcal {H}}}
\newcommand{\calM}{{\mathcal {M}}}
\newcommand{\calP}{{\mathcal {P}}}
\newcommand{\calC}{{\mathcal {C}}}
\newcommand{\calL}{{\mathcal L}}
\newcommand{\spa}{\operatorname{span}}

\newcommand{\F}{\mathcal F}

\newcommand{\gotm}{{\mathfrak m}}

\newcommand{\ve}{\varepsilon}

\newcommand{\lan}{\langle}
\newcommand{\ran}{\rangle}

\newcommand{\calN}{{\mathcal N}}

\newcommand{\kk}{{\mathfrak {K}}}

\def\Koniec{\hbox to\hsize{\hfil$\diamond$}}

\def\1{\mathbf 1}

\newtheorem{theorem}{Theorem}[section]
\newtheorem{corollary}[theorem]{Corollary}
\newtheorem{lemma}[theorem]{Lemma}
\newtheorem{proposition}[theorem]{Proposition}

\theoremstyle{definition}
\newtheorem{definition}{Definition}[section]
\theoremstyle{remark}

\numberwithin{equation}{section}

\usepackage[normalem]{ulem}
\usepackage[pagewise]{lineno}%\linenumbers

\definecolor{blue-violet}{rgb}{0.54,0.17,0.89}
\definecolor{amethyst}{rgb}{0.6,0.4,0.8}
\definecolor{darkviolet}{rgb}{0.58, 0.0, 0.83}
\definecolor{darkgreen}{rgb}{0,.4,0}
\definecolor{mixedgreen}{rgb}{0.3,0.6,00}
\definecolor{bananayellow}{rgb}{1.0, 0.88, 0.21}
\definecolor{arylideyellow}{rgb}{0.91, 0.84, 0.42}
\definecolor{bananamania}{rgb}{0.98, 0.91, 0.71}

\newcommand{\REMOVEsk}[1]%
           {{\color{magenta}\sout{#1}}}

\allowdisplaybreaks

\begin{document}

%\title{Mercer Large-Scale Kernel Machines}
\title{Mercer  Kernel Machines from a Ridge Function Perspective}

\author{Karol Dziedziul}
\address{
Faculty of Applied Mathematics,
The Gda\'nsk University of Technology,
ul. G. Narutowicza 11/12,
80-952 Gda\'nsk, Poland}

\email{karol.dziedziul@pg.edu.pl}

\author{Sergey Kryzhevich}
\address{
Faculty of Applied Mathematics,
The Gda\'nsk University of Technology,
ul. G. Narutowicza 11/12,
80-952 Gda\'nsk, Poland}

\email{sergey.kryzhevich@pg.edu.pl}

\author{Paweł Wieczyński}

\begin{abstract}
In order to present Mercer large-scale kernel machines from the perspective of ridge functions, we first recall the results of Lin and Pinkus from Fundamentality of Ridge Functions. We then revisit the main result of Rahimi and Recht’s paper Random Features for Large-Scale Kernel Machines (2008) from the standpoint of approximation theory. In particular, we study which kernels can be approximated by finite sums of products of cosine functions  and we identify the main obstacles inherent in this approach. Finally, we apply these results to image processing, using a one-vs-rest procedure, and to the setting of adaptive algorithms.

\end{abstract}

\maketitle

\begin{quote}
\end{quote}

\begin{quote}
\end{quote}

\bigskip

\section{Introduction}

In the context of Reproducing Kernel Hilbert Spaces (RKHS), a \textit{kernel} is a function \( k: X \times X \to \mathbb{R} \) (or \(\mathbb{C}\)) that maps pairs of elements from a set \( X\subset \R^d \) into the real (or complex) numbers. Here, we consider only real-valued functions. The function must satisfy the following properties:

1. \textit{Symmetry}: \( k(x, y) = k(y, x) \) for all \( x, y \in X \).

2. \textit{Positive definiteness}: For any finite set of points \( \{x_1, x_2, \ldots, x_n\} \subset X \),  any vector \( \mathbf{c} = (\gotc_1, \gotc_2, \ldots, \gotc_n)^\top \in \mathbb{R}^n \),
   \[
   \mathbf{c}^\top G_n \mathbf{c} = \sum_{i=1}^n \sum_{j=1}^n \gotc_i \gotc_j k(x_i, x_j) \geq 0.
   \]
   In other words the Gram matrix \( G_n \) defined by \( G_n =[ k(x_i, x_j)]_{1\leq i,j\leq n} \)is semi-positive.
  In our approach, we mainly assume that  $X=\R^d$ or $X=[0,1]^d$. Generally we assume that $k$ is continuous. For first reading  we propose \cite{SVM}.
   
     If there exists a function $\kk$ such that  $k(x,y)=\kk(x-y)$ we say that the kernel $k$ is {\it shift invariant} or {\it translation invariant kernel.} In applications of machine learning, we encounter universal kernels, which are studied by Ch. A. Micchelli, Y. Xu, and H. Zhang \cite{Micc}. An example of such 
kernel is the Gaussian kernel, $k(x,y)=\kk_G(x-y)$ , where $\kk_G$ is
the multivariate Gaussian density function (up to the constant)  given by:
\[ \kk_G(x) \asymp  \exp\left(-\gamma {x}^\top {x}\right), x\in \R^d,
 \]
 where $\gamma>0$. 
 A universal kernel is a continuous function  for which, for any compact subset $K\subset \R^d$, the set
 $\spa \{k(\cdot,y) : y \in K\}$ is dense in $C(K)$ under the maximum norm.

In this paper, we introduce ridge kernels depending on 
 $\vartheta$  which includes both the set of parameters $(\{c_j,b_j,w_j\}_{j=1}^\gotm)$, $c_j>0$, $b_j\in\R$, $w_j\in \R^d$ and  consider  a set of functions $G\subset C(\R,\R)$ 
\begin{equation}
\label{RK}
k(x,y|\vartheta)=\sum_{j=1}^\gotm c_j g_j(\lan w_j,x\ran+b_j)g_j(\lan w_j,y\ran+b_j),\quad c_j\geq 0,
\end{equation}
such that $g_j\in G, j=1,\ldots,\gotm$. 
We pay a special attention to the case of the singleton set
$G=\{\sigma\}$, $\sigma\in C(\R)$
 \begin{equation}
\label{kRR}
k(x,y|\vartheta)=\sum_{j=1}^\gotm c_j \sigma(\lan w_j,x\ran+b_j)\sigma(\lan w_j,y\ran+b_j).
\end{equation}
In case $\sigma=\cos$ we get  Rahimi-Recht kernel \cite{Rahimi}.
It is evident that  for fixed $\vartheta$, $k$ is not a universal kernel. 
But if we stat to change $\vartheta$ we obtain a family of kernels. Why might we still need such kernels? Sometimes we require universal tools, but other times, we need particular tools  to do a better job, i.e. to approximate functions of some specific types. This motivates the following definition.
\begin{definition}\label{universalfamily}
We call a family $\F$ of kernels  {\it universal}  if, for any compact set $K\subset \R^d$ and any continuous function $f\in C(K)$, any $\varepsilon>0$, there exists $n\in {\mathbb N}$, $k\in \F$ and coefficients $\{a_j\}_{j=1}^n\subset  \R$ and $\{x_j\}_{j=1}^n\subset K$ such that
\[
\sup_{x\in K}|f(x)-\sum_{j=1}^n a_j k(x,x_j)|\leq \varepsilon.
\]
Members of $\F$ are called particular kernels. 
\end{definition}

Note  that in case $k$ is a universal kernel then $\F=\{k(x,y)\}$ is a universal family.
In Section \ref{universal} we prove among other results that family of kernels of form as \eqref{kRR}, i.e.
\[
\F=\{k(x,y|\vartheta):\vartheta=\{(c_j,b_j,w_j)_{j=1}^\gotm\}, c_j>0, b_j,w_j\in\R^d\}
\] is universal if and only if $\sigma$ is not a polynomial.
It is a  consequence of universal approximation theorem, We use arguments similar to ones from proof of \cite[Theorem 5.1]{Pinkus2}.
 To  bound the space of parameters, for instance to assure that $w_j\in \Sp^{d-1}$ we should to assume some additional assumption on  $\sigma$.

The contribution of this paper is as follows.   At first we examine the approximation properties of  ridge functions   $g_1(\lan w, x\ran)g_2(\lan w, y\ran)$ where $w\in \R^d, g_1,g_2\in C(\R,\R)$ in the topology of uniform convergence on compact sets.
This calls is larger than ridge kernels.
 We show that  
\[
\overline{\spa\,\{ g_1(\lan w, x\ran)g_2(\lan w, y\ran): w\in \R^d, g_1,g_2\in C(\R,\R) \}}\neq
C(\R^{2d},\R),
\]
see Theorem  \ref{l1.2}.  
In Theorem \ref{almost} we demonstrate the geometric structure of space of polynomials contained in
$\overline{\spa\,\{ g_1(\lan w, x\ran)g_2(\lan w, y\ran): w\in \R^d, g_1,g_2\in C(\R,\R) \}}$.
It is  progress comparing \cite{Pinkus1}.
On the over hand
  Rahimi and Recht show that Rahimi-Recht  kernels approximates any continuous , shift-invariant kernel
in the topology of uniform convergence on compact sets.   Rahimi and Recht result is proved in \cite{Rahimi} using Bochner's theorem.
Using Universal Approximation Theorem  we obtain 
parallel result Theorem \ref{th_appr} for ridge functions of form $g_1(\lan w, x\ran)g_2(\lan w, y\ran)$.
 
 Secondly part is connected with ridge kernels. From general perspective this approach ties neural networks  with machine learning.  We show that  kernels \eqref{kRR} form 
 universal family. 
 This motivates to first numerical simulation. 
We  show how to use particular kernels as a set of filters in the spirit of  "one-vs-rest".

Next, we connect machine learning with classical nonparametric estimation.
In the framework developed by Smale and Cucker, all computations are performed in a Hilbert space associated with a fixed Mercer kernel, namely the corresponding RKHS. However, the true regression function need not belong to this space. As a consequence, what is effectively estimated is not the regression function itself, but only its projection—or shadow—onto the RKHS. The defect function is typically estimated within this framework; see \cite{Cucker} (Theorem A and Theorem B).

The main limitation of this approach is that it provides only a surrogate of the regression function. In particular, the estimation error of the regression function itself is not directly controlled, but only the error of its representation within the RKHS. To overcome this difficulties there are straight  methods of estimation, introduced in \cite{Binev}. In paper a rate of estimation is optimal \cite{Dziedziul}. Both papers  invokes  classical methods of estimation.

By using ridge kernels, we aim to obtain a more accurate estimation of the regression function while remaining within the RKHS perspective. The motivation for this approach is rooted in classical nonparametric estimation theory, which provides optimal convergence rates for regression function estimation; see \cite{Stone1982}, \cite{Stone1985}.

More recent results on nonparametric estimation using neural networks \cite{Bagirov} and deep neural networks \cite{Kohler} are also highly relevant from our perspective. These results demonstrate that, despite the theoretical and algorithmic challenges, interpolation-based methods can still convey meaningful information about the underlying regression function.

Therefore, we propose to bridge these two approaches by transforming ridge kernels into an adaptive algorithm that progressively refines the effective hypothesis space.

\section*{\large{Approximation properties of ridge functions}}
\section{Preliminaries}
In this chapter, we rely on the results of \cite{Pinkus1}, in particular on Theorem 2.1. Unless stated otherwise, we use the notation of \cite{Pinkus1}. Let $\Omega$ be a subset of the set of real matrices of dimension $d\times \tau$. Consider the functional space
\[
\calM(\Omega)=\mbox{\rm span}\,\{ g(Az): A\in \Omega, g\in C(\R^d,\R)\}
\]
where $z\in \R^\tau$. Let $L(A)$ denote the span of rows of $A$ and
\[
L(\Omega)=\bigcup_{A\in \Omega} L(A).
\]
\begin{definition}
A sequence of function $\{f_j\}_{j=1}^\infty\subset C(\R^d)$ is said to converge compactly
to a function $f\in C(\R^d)$ is for every compact set 
$K\subset \R^d$ 
\[
\lim_{j\to \infty}\sup_{x\in K}|f_j(x)-f(x)|=0.
\]
 The 
convergence on compact sets in $C(\R^d)$ gives compact-open topology. We call the topology of uniform 
convergence on compact sets.
\end{definition}
\begin{theorem}\label{Pinkus}
The linear space  $\calM(\Omega)$ is dense in $C(\R^{\tau},\R)$ endowed with the topology of uniform
convergence on compact sets if and only if the only homogeneous polynomial of $\tau$ variables which vanishes identically on $L(\Omega)$ is the zero polynomial.
\end{theorem}

Following \cite{Pinkus1}, we denote by $\calP_\Omega$ the set of all polynomials, vanishing identically on $L(\Omega)$, i.e.
\[
\calP_\Omega =\{p\in \Pi^{\tau} : p|_{L(\Omega)}=0 \},
\]
where $\Pi^{\tau}$ is the set of all algebraic polynomials of $\tau$ variables. By $\Pi^\tau_k\subset \Pi^\tau$, we denote the set of polynomials of total degree $k$. 

According to \cite[Theorem 4.1]{Pinkus1}, in order to find the set $\overline{\calM(\Omega)}$, it is sufficient to find the polynomials $q$ such that for all polynomials $p$ vanishing on $L(\Omega)$ we have 
\begin{equation}
\label{poly}
p(D) q=0.
\end{equation}
Here $D$ is the gradient vector; for any multiindex $m=(m_1,\ldots,m_\tau)$ we can define $$D^m=\dfrac{\partial^{|m|}}{\partial z^m}.$$
Speaking of the closure of a set of functions (for instance, $\overline{\calM(\Omega)}$), we always mean the topology of uniform convergence on compact sets. Introduce the set of all polynomials satisfying \eqref{poly}:
\[
\calC_1=\mbox{\rm span}\,\{q\in \Pi^\tau: p(D) q=0 \quad \hbox{for all}\quad p\in \calP_\Omega\}.
\]
Moreover, let
\[
\calC_2=\mbox{\rm span}\,\{ q(z)=(\lan z, a\ran)^l:  l\in \Z_+, z\in \R^{\tau}, a\in \mathcal N \}
\]
where
\[
 \NN:=\mathrm{ker}\, \calP_\Omega:=\bigcap_{p\in \calP_\Omega} \mathrm{ker}\, p.
\]
Let 
\[
\calC_3=\mbox{\rm span}\,\{ f(\lan z, a\ran):  f\in C(\R,\R), a\in \mathcal N \}
\]
Now we can reformulate \cite[Theorem 4.1]{Pinkus1}.

\begin{theorem}\label{Momega}
In the topology of uniform convergence on compact sets, we have
\[
\overline{\calM(\Omega)}=\overline{\calC_1}=\overline{\calC_2}=\overline{\calC_3}.
\]
\end{theorem}
\begin{proof} The Weierstrass Theorem implies that
$\overline{\calC_2}= \overline{\calC_3}$.
It follows from the proof of \cite[Theorem 4.1]{Pinkus1} that 
\[
\overline{\calM(\Omega)}=\overline{\calC_1}=\overline{\calC_3}.
\]
\end{proof}
Consider a specific set $\Omega_S$ of $2\times 2d$ matrices of the form
\begin{equation}\label{matrixa}
A=\left[
\begin{matrix}
w & 0\cr
0 & -w
\end{matrix}
\right],\qquad 0,w\in \R^d.
\end{equation}
We will show that
\[
\calM(\Omega_S)=  \spa\,\{ g(Az): A\in \Omega_S, \quad g\in C(\R^2,\R) \},
\]
is not dense in $C(\R^{2d},\R)$ endowed with the topology of uniform convergence on compact sets.

Let the subspace $L(A)\subset \R^{2d}$ be the span of two rows of $A$ and, as we defined above,
\begin{equation}
\label{def1}
L(\Omega_S):=\bigcup_{A\in \Omega_S} L(A).
\end{equation}
Denote
\begin{equation}
\label{def2}
 \NN_S=\mathrm{ker}\, \calP_{\Omega_S}:=\bigcap_{p\in \calP_{\Omega_S}} \mathrm{ker}\, p,
\end{equation}
where
\begin{equation}
\label{def3}
\calP_{\Omega_S} =\{p\in \Pi^{2d} : p|_{L(\Omega_S)}=0 \}
\end{equation}
and $\Pi^{2d}$ is a space of polynomials of $2d$ variables.

\section{Polynomials vanishing on $L(\Omega_S)$}

Following the approach proposed by Lin and Pinkus \cite{Pinkus1}, we demonstrate that the closure $\overline{\calM(\Omega_S)}$ is not $C(\R^2,\R)$.

Let $H^{2d}_k$ denote the set of homogeneous polynomials of $2d$ variables of total degree $k$ i.e.
\[
H^{2d}_k=\left\{ \sum_{|m|=k} a_m z^m, \quad z\in \R^{2d}, m\in \Z^{2d}_+\right\},
\]
where $z^m=z_1^{m_1}\cdots z_{2d}^{m_{2d}}$ and $|m|=m_1+\cdots+m_{2d}=k$. Let
\[
\calH_{\Omega_S,k} =\{p\in H^{2d}_k : p|_{L(\Omega_S)}=0 \}.
\]
For any $s\in \Z_+^d$, $|s|=k$ and $0\leq l\leq k$ we define the set
\[
\Delta_{s,l}=\big\{\overline{m} \in [0,s_1]\times \cdots \times [0,s_n]\cap \Z_+^d: |\overline{m}| =l\big\}.
\]
For any $0<l<k$, $|s|=k$ we define a subspace of $H^{2d}_k$
\[
\calH(\Delta_{s,l})=\{p\in H^{2d}_k : p=P_{s,l} \},
\]
where for $z=(x,y)$, $x,y\in \R^d$, $m=(\overline{m},s-\overline{m})\in \Z_+^{2d}$ and $\overline{m}\in \Z_+^d$
\[
P_{s,l}(x,y)= \sum_{\overline{m}\in \Delta_{s,l}} a_{\overline{m},s-\overline{m}} x^{\overline{m}}
y^{s-\overline{m}}, 
\]
and the coefficients of polynomial $P_{s,l}$ are such that  
\[
\sum_{\overline{m}\in \Delta_{s,l}} a_{\overline{m},s-\overline{m}}=0.
\]

\begin{theorem}\label{l1.2}
Let $\Omega_S$ be a family of $2\times 2d$ matrices of the form \eqref{matrixa}, $d>1$. Then 
\begin{equation}
\label{equality}
\bigcup_{|s|=k,0<l<k} \calH(\Delta_{s,l}) = \calH_{\Omega_S,k}
\end{equation}
Consequently, $\calM(\Omega_S)$ is not dense in $C(\R^{2d},\R)$ with the topology of uniform convergence on compact sets.
\end{theorem}

\begin{proof}

Let $P$ be any homogeneous polynomial of $2d$ variables of total degree $k$, i.e.
for $m=(\overline{m},\underline{m})$, $\overline{m},\underline{m} \in \Z_+^d$, $x,y\in \R^d$, $z=(x,y)$
\[
P(z)=P(x,y)=\sum_{|m|=k} a_m x^{\overline{m}} y^{\underline{m}}=\sum_{|m|=k} a_m z^m.
\]
Here $\Z_+$ stands for the set of all non-negative integers. We use the multivariate notation. 
Let $A\in \Omega_S$. Then 
\[
L(A)=\{(x,y)=\alpha (w,0)+\beta (0,w), \alpha,-\beta\in \R \}.
\]
Hence for any vector $(x,y)\in L(A)$,
\[
P(x,y)=\sum_{|m|=k} a_m
 x^{\overline{m}} y^{\underline{m}}=\sum_{|m|=k} a_m \alpha^{|\overline{m}|} \beta^{|\underline{m}|} w^{\overline{m}+\underline{m}}.
\]
A linear operator $E:\Z^{2d} \to \Z^{2d}$
\begin{equation}
\label{linear}
E(m)=E(\overline{m},\underline{m})=(\overline{m},\overline{m}+\underline{m})=(\overline{m},s), \quad s(m)=s=\overline{m}+\underline{m}
\end{equation}
is isomorphism. Let $B=E^{-1}$. Hence $B(\overline{m},s)= (\overline{m},s-\overline{m})$ then, using $\overline{m}=(m_1,\ldots,m_d)$, $(x,y)=\alpha (w,0)+\beta (0,w)$
\[
P(x,y)=\sum_{|s|=k,s\in \Z_+^d} \sum_{m_1=0}^{s_1}\cdots \sum_{m_n=0}^{s_d}
a_{B(\overline{m},s)}\alpha^{|\overline{m}|} \beta^{k-|\overline{m}|} w^{s}.
\]
Observe that $\# \Delta_{s,l}=1$ if $|s|=l$ or $l=0$ or there is only one non-zero entry $s_j$. In all other cases, $\# \Delta_{s,l} >1$. Here $\#$ stands for the cardinality of a set. Then, for all $w\in {\R^d}$, $\alpha,\beta\in \R$,
$(x,y)=\alpha (w,0)+\beta (0,w)$,
\begin{equation}
\label{cond0}
P(x,y)=\sum_{|s|=k,s\in \Z_+^d} \sum_{l=0}^k \sum_{\overline{m}\in \Delta_{s,l}}
a_{B(\overline{m},s)}\alpha^{l} \beta^{k-l} w^{s}=0
\end{equation}
if and only if for all $s,l\leq |s|$
\begin{equation}
\label{cond}
\sum_{\overline{m}\in \Delta_{s,l}}
a_{B(\overline{m},s)}=0.
\end{equation}
This is $(k+1)\binom{d+k-1}{k}$ equations. Observe that $\binom{d+k-1}{k}$  is a dimension of space of homogeneous polynomials of $d$ variables of degree $k$. We see that if $\#\Delta_{s,l}>1$ then there exist non-zero polynomials that vanish $L(\Omega_S)$. 

Taking to account Eq.\,\eqref{cond}, we see that both
$\calH_{\Omega,k} $ and $\calH(\Delta_{s,l})$ are linear spaces, $\calH(\Delta_{s,l}) \subset
\calH_{\Omega,k}$. If $\#\Delta_{s,l}>1$ then $\dim(\calH(\Delta_{s,l}))=\#\Delta_{s,l}-1$. To see this, fix one coefficient $a_{\overline{m_0},s-\overline{m_0}}=-1$ and change the others to zero, but one equals one.
The equality \eqref{equality} is a consequence of an equivalence of \eqref{cond0} and \eqref{cond}.
Now a conclusion that 
$\calM(\Omega_S)$ is not dense in $C(\R^{2n},\R)$ with the topology of uniform convergence on compact sets follows from Theorem \ref{Pinkus}. Indeed we showed that exists
a non zero homogeneous polynomial of $\tau=2n$ variables which vanishes identically on $L(\Omega_S)$. 
\end{proof}

\section{Characterization of polynomials in $\overline{\calM(\Omega_S)}$}

Let $H^{2d}$ be a set of all  homogeneous polynomials of $2d$ variables and
\[
\calH_{\Omega_S} =\{p\in H^{2d} : p|_{L(\Omega_S)}=0 \},
\]
\[
\calP_{\Omega_S} =\{p\in \Pi^{2d} : p|_{L(\Omega_S)}=0 \},
\]
\[
 \NN_S=\mathrm{ker}\, \calP_{\Omega_S}:=\bigcap_{p\in \calP_{\Omega_S}} \mathrm{ker}\, p.
\]

\begin{lemma}\label{NS}
Then
\[
\NN_S=\mathrm{ker}\, \calH_{\Omega_S}:=\bigcap_{p\in \calH_{\Omega_S}} \mathrm{ker}\, p.
\]
\end{lemma}

\begin{proof}
Inclusion $\NN_S\subset \hbox{ker}\, \calH_{\Omega_S}$ is obvious. Let $q\in \calP_{\Omega_S}$,
\[
q(z)=\sum_{|m|\leq \rho} a_m z^m=\sum_{k=0}^\rho q_k(z),\qquad z\in \R^{2d},
\]
where $q_k$ is a homogeneous polynomial
\[
q_k(z)=\sum_{|m|=k} a_m z^m.
\]
Then for all $t\in \R$ and all $z\in L(\Omega_S)$, $tz\in L(\Omega_S)$
\[
q(tz)=\sum_{k=0}^\rho t^k q_k(z)=0.
\]
Hence for all $k=0,\ldots,\rho$   
\[
q_k(z)=0, z\in L(\Omega_S),\qquad q_k\in \calH_{\Omega_S}.
\]
 Hence
 \[
  \bigcap_{k=1}^\rho \hbox{ker}\, q_k= \hbox{ker}\, q.
 \]
 %\bigcap_{p\in \calH_{\Omega_S}} \hbox{ker}\, p \subset
 This statement finishes the proof. 
\end{proof}

\begin{theorem}\label{almost}
Let $\Omega_S$ be a family of $2\times 2d$ matrices of the form \eqref{matrixa}, $d>1$. Then
\[
L(\Omega_S)=\overline{L(\Omega_S)}=\NN_S.
\]
\end{theorem}
\begin{proof}
It follows from definitions \eqref{def1}-\eqref{def3} that
\begin{equation}\label{lns}
L(\Omega_S) \subset \overline{L(\Omega_S)}\subset \NN_S.
\end{equation}
By Lemma \ref{NS}, definition of $\calH_{\Omega_S,k}$ and Theorem \ref{l1.2}, we obtain
\[
\NN_S=\mathrm{ker}\,\calH_{\Omega_S}\subset \bigcap_{|s|=k, 0<l<k} \mathrm{ker}\, \calH(\Delta_{s,l})=\mathrm{ker}\,\calH_{\Omega_S,k} 
\]
Now it suffices to prove that for sets $\calH(\Delta_{s,l})$ large enough
\[
\bigcap \mathrm{ker}\, \calH(\Delta_{s,l}) \subset  L(\Omega_S).
\]
Here the polynomials identifying $\NN_S$ are given below in Step 1- Step 3.

Let us fix $s\in \Z_+^d$, $|s|=k$. Let $\#\Delta_{s,l}>1$. From Theorem \ref{l1.2}, $\calH(\Delta_{s,l}) \subset\calH_{\Omega_S,k}$. 
For a fixed multiindex $\kappa \in \Delta_{s,l}$ the test set of polynomials is constructed as follows: $b_{s,l,\kappa,m}$ is for
\[\begin{array}{c}
m=(m_1,m_2,\ldots, m_d)\neq \kappa=(\kappa_1,\ldots,\kappa_d), \quad |m|=l, m\in \Delta_{s,l}, 
\\ s=(s_1,s_2,\ldots, s_d), |s|=k, \quad  \quad \mbox{for all} \quad 0\le j\le d,
\end{array}\]
\[
b_{s,l,\kappa,m}(x,y)=x^{\kappa}y^{s-\kappa}- x^{m}y^{s-m},
\]
where
\[
x=(x_1,\ldots,x_d),\quad y=(y_1,\ldots,y_d).
\]
Let $(x,y)\in \NN_S$. Then $b_{s,l,\kappa,m}(x,y)=0$ for all $ b_{s,l,\kappa,m} \in \calH(\Delta_{s,l}) $
if and only if for all $s$, $m$ as above
\begin{equation}\label{bslm}
x_1^{\kappa_1} x_2^{\kappa_2}\cdots x_d^{\kappa_d}  y_1^{s_1-\kappa_1} y_2^{s_2-\kappa_2}\cdots y_d^{s_d-\kappa_d}=x_1^{m_1}x_2^{m_2}\cdots x_d^{m_d} y_1^{s_1-m_1} y_2^{s_2-m_2}\cdots y_d^{s_d-m_d}.
\end{equation}
We would like to show that $(x,y)\in L(\Omega_S)$, i.e. there are $\alpha,\beta\in \R$ and a vector $w\in \R^d$ such that
\begin{equation}\label{LO}
(x,y)=\alpha(w,0,\ldots,0)+\beta (0,\ldots,0,w).
\end{equation}

\textbf{Step 1.} We show that if $x_1\neq 0,\ldots, x_d\neq 0$ then either all $y_i$ are zero, all  $y_j$ are non-zero. Take  $\kappa=(1,0,\ldots,0)$, $s=(1,1,0,\ldots,0)$,
$m=(0,1,0,\ldots,0)$. Then
\[
b_{s,l,\kappa,m}(x,y)=x_1y_2-x_2y_1=0.
\]
Hence $y_2=\frac{x_2}{x_1} y_1$. By a similar way we can show that $y_j=\frac{x_j}{x_1} y_1$ for $j=3,\ldots,d$.
Hence we prove the desired statement.

\textbf{Step 2.} Thus if all coefficients $x$ and $y$ are non-zero then
\[
x_1^{m_1-\kappa_1} x_2^{m_2-\kappa_2}\cdots x_d^{m_d-\kappa_d}  = y_1^{m_1-\kappa_1} y_2^{m_2-\kappa_2}\cdots y_d^{m_d-\kappa_d}.
\]
\[
1=\Big(\frac{x_1}{y_1}\Big)^{m_1-\kappa_1} \Big(\frac{x_2}{y_2}\Big)^{m_2-\kappa_2}\cdots \Big(\frac{x_d}{y_d}\Big)^{m_d-\kappa_d}.
\]
Without loss of generality, we can assume that $m_1<s_1$. Otherwise, we can increase $s_1$.
Take $m_1=\kappa_1+1$ and $m_i=\kappa_i-1$ and $m_j=\kappa_j$ for $j\neq 1$ and $j\neq i$. Then 
\[
1=\Big(\frac{x_1}{y_1}\Big)^{-1} \Big(\frac{x_i}{y_i}\Big).
\]
Since $i$ is selected arbitrarily, we get 
\begin{equation}\label{xty}
x_i=ty_i, \qquad i=1,2,\ldots,d.
\end{equation}
Consequently, we get \eqref{LO}. 

\textbf{Step 3.} Now let us assume that $(x,y)\in \NN_S$ is such that  $x_1=\cdots =x_i=0$ and $x_{i+1}\neq 0, \ldots, x_d\neq 0$. We want to show that $y_1=\cdots =y_i=0$.
Take $\kappa_1=\cdots=\kappa_i=0,\kappa_{i+1}=1,\ldots,\kappa_d=1$,
$s_1=1,s_2=\cdots = s_i=0,s_{i+1}=1,\ldots,s_d=1$, $m_1=1,m_2=\cdots =m_i=m_{i+1}=0$, $m_{i+2}=\cdots =m_d=1$.
Since $b_{s,l,\kappa,m}(x,y)=0$ then $y_1=0$ in similar way we get $y_2=\cdots =y_i=0$. Now by Step 1, 
$y_j=\frac{x_j}{x_1} y_1$ for $j=2,\ldots,d$. Hence, by Step 2, there is $t$ such that
\[
t x_i=y_i, \qquad i=1,\ldots,d.
\]
Then \eqref{LO} is satisfied. 

\end{proof}

By  Theorem  \ref{almost} and by Theorem \ref{Momega} we obtain a more explicit characterization of $\overline{\calM(\Omega_S)}$ by polynomials $\calC_2$, i.e.

\begin{corollary}
Let $\Omega_S$ be a family of $2\times 2d$ matrices of the form \eqref{matrixa}, $d>1$. Then
\[
\overline{\calM(\Omega_S)}=\overline{\calC_2}
\]
where we take $\tau=2d$
\[
\calC_2=\mbox{\rm span}\,\{ q(z)=(\lan z, a\ran)^l:  l\in \Z_+, z\in \R^{2d}, a\in L(\Omega_S)\}.
\]
\end{corollary}
This shows immediately that for $d=2$
such kernels
$x=(x_1,x_2), y=(y_1,y_2)$
\begin{equation}
\label{not1}
k(x,y)=(x_1^2+x_2^2)(y_1^2+y_2^2)
\end{equation}
or
\begin{equation}
\label{not2}
k(x,y)=(x_1^2-x_2^2)(y_1^2-y_2^2).
\end{equation}
do not are contained in $\calC_2$.
Since functions $g=g_1\otimes g_2$ are dense in $C(\R^{2},\R)$ with the topology of uniform convergence on compact sets, we get that in the topology of uniform convergence on compact sets, 
\begin{equation}
\label{tensor}
\overline{\calM(\Omega_S)}=\overline{\spa\,\{ g_1(\lan w, x\ran)g_2(\lan w, y\ran): w\in \R^n, g_1,g_2\in C(\R,\R) \}}.
\end{equation}
Note that all consider families $\F\subset \overline{\calM(\Omega_S)}$.

\section{Approximation property }\label{UATsec}

The following theorem is similar to the RR result, but there are some differences. The theorem below states that a symmetric function
  $ k: \R^{2d}\to \R$ has an approximant given by \eqref{kRR}, where $\sigma=\cos$. 
 but without any restrictions on the coefficients $\{c_j\}$. 
It is counterpart of mentioned Rahimi Recht's theorem.

\begin{theorem}\label{th_appr} Let $ { k}: \R^{2d}\to \R$ be a continuous function such that there exists a real valued continuous function $\tilde{{k}}:\R^d\to \R $ such that  ${k}(x,y)=\tilde{{ k}}(x-y)$.
Then $k\in \overline{\calM(\Omega_S)}$.

In addition, for any compact set $K\subset \R^{2d}$, and for every $\varepsilon>0$ there exist $\gotm_1,\gotm_2$, coefficients $c_j\in \R$ and
parameters $t_k\in \R$, $w_j\in \R^n$ such that for all $(x,y)\in K $
\begin{equation}
\label{calka2}
\left|{k}(x,y)- \sum_{j=1}^{\gotm_1} \sum_{k=1}^{\gotm_2} c_j \cos(\lan w_j, x\ran+t_k)\cos(\lan w_j, y\ran+t_k)\right|\leq \varepsilon.
\end{equation}
\end{theorem}

\begin{proof}
Let $k$ be a symmetric function, such that ${k}(x,y)=\tilde{{k}}(x-y)$ (as in the statement of the theorem). Since $K\subset \R^{2d}$ is the compact set then there is a compact set $\tilde{K}\subset \R^d$ such that if $(x,y)\in K$ then $x-y\in \tilde{K}$ and  $y-x\in \tilde{K}$.       
Let  $\sigma:\R\to \R$ be a so-called activating function, which is continuous and not a polynomial.  By \cite[Proposition 3.7]{Pinkus2}, the space
\[
\spa\{\sigma(\lambda t- \theta): \lambda,b\in \R  \},\qquad t\in \R,
\]
is dense in $C(\R,\R)$ endowed with the topology of uniform convergence on compact sets.  Then by \cite[Proposition 3.3]{Pinkus2}, 
\[
\spa\{\sigma( \lan w,x\ran-b): w\in \R^d, b\in \R \},
\]
is dense in $C(\R^d)$ with the topology of uniform convergence on compact sets. Hence, we have
 \[
\tilde{{ k}}\in \overline{\spa}\{\sigma( \lan w,x\ran-b): w\in \R^d, b\in \R \}
\]
Consequently, for all $\varepsilon>0$ and $u\in \tilde{K}$, there exist constants $m$, $c_j\in \R$, $b_j\in \R$ 
\begin{equation}
 \label{rep0}
\left|\tilde{{ k}}(u)- \sum_{j=1}^\gotm c_j \sigma (\lan w_j, u\ran-b_j)\right|\leq\varepsilon.
 \end{equation}
Let  $a_j=(w_j,-w_j)\in \R^{2d}$. Then  for $(x,y)\in K\subset \R^{2d}$
 \begin{equation}
 \label{rep}
\left|{k}(x,y)- \sum_{j=1}^\gotm c_j \sigma (\lan a_j, (x,y)\ran-b_j)\right|\leq\varepsilon.
 \end{equation}
Note that $a_j=(w_j,-w_j)\in L(\Omega_S)$.
By \cite[Theorem 4.1 (2)]{Pinkus1} (see  also Theorem \ref{Momega}) and taking into account the inclusion \eqref{lns}, we get ${ k}\in \overline{\calM(\Omega_S)}$.

By Eq.\, \eqref{rep}, we get that for $\sigma=\cos$ and any $\varepsilon>0$ there exist constants $m$, $c_j\in \R$, $b_j\in \R$ and vectors $a_j=(w_j,-w_j)\in \R^{2d}$ such that
\[
\left|{ k}(x,y)-  \sum_{j=1}^\gotm c_j \Big(\cos (\lan w_j,x-y\ran)\cos(b_j)+\sin (\lan w_j,x-y\ran)\sin(b_j)\Big)\right|\leq\varepsilon,\qquad (x,y)\in K.
\]
Now we show that also
\begin{equation}
\label{claim}
\left|{ k}(x,y)-  \sum_{j=1}^\gotm c_j \Big(\cos (\lan w_j,x-y\ran)\cos(b_j)\Big)\right|\leq\varepsilon,\qquad (x,y)\in K.
\end{equation}
Indeed,
\[
\begin{aligned}
&\quad 2\left|{ k}(x,y)-  \sum_{j=1}^\gotm c_j \Big(\cos (\lan w_j,x-y\ran)\cos(b_j)\Big)\right|
\\
&\leq \left|2{ k}(x,y)-  \sum_{j=1}^\gotm c_j \Big(2\cos (\lan w_j,x-y\ran)\cos(b_j) \right.
\\
&\qquad\quad \left.-\sin (\lan w_j,x-y\ran)\sin(b_j)
+\sin (\lan w_j,x-y\ran)\sin(b_j\Big)\right|
\\
&\leq \left|{ k}(x,y)-  \sum_{j=1}^\gotm c_j \Big(\cos (\lan w_j,x-y\ran)\cos(b_j)-\sin (\lan w_j,x-y\ran)\sin(b_j)\Big)\right|
\\
&\qquad+\left|{ k}(x,y)-  \sum_{j=1}^\gotm c_j \Big(\cos (\lan w_j,x-y\ran)\cos(b_j)+\sin (\lan w_j,x-y\ran)\sin(b_j)\Big)\right|
\end{aligned}
\]
Take $(x,y)\in K$. Then, by definition of $\tilde K$, we have $x-y\in \tilde{K}$ and $y-x\in \tilde{K}$. Since ${ k}(x,y)={ k}(y,x)$, and since the function $\cos$ is even and $\sin$ is odd, we get the claim \eqref{claim}.  It follows from trigonometric formulas, that  
\[
\int_0^{2\pi} 2\cos(\alpha+t)\cos(\beta+t) dt=
\int_0^{2\pi} \left( \cos(\alpha-\beta)+\cos(\alpha+\beta+2t)\right) dt=2\pi \cos(\alpha-\beta).
\]
Putting $\alpha=\lan w_j, x\ran, \beta=\lan w_j, y\ran$, we obtain that for $\epsilon>0$ there are $\gotm$, $\tilde{c}_j\in \R$, $\theta_j\in \R$ and vectors $w_j\in \R^n$
\begin{equation}
\label{calka}
\left|{ k}(x,y)- \sum_{j=1}^\gotm \tilde{c}_j \int_0^{2\pi} \cos(\lan w_j, x\ran+t)\cos(\lan w_j, y\ran+t) dt
\right|\leq\varepsilon, \qquad (x,y)\in K.
\end{equation}
Now we use the uniform approximation of the integral 
\[
\int_0^{2\pi} \cos(a+t)\cos(b+t) dt,
\] 
on the interval $[0,2\pi]$. Namely, for any $a,b\in \R$ and any $\varepsilon>0$ there is $\gotm_2>0$ such that for 
$\{t_k= 2\pi k/\gotm_2,\quad k=1,\ldots, \gotm_2\}$ we have
\[
\left| \int_0^{2\pi} \cos(a+t)\cos(b+t) dt-\frac{1}{\gotm_2}\sum_{k=1}^{\gotm_2} \cos(a+t_k)\cos(b+t_k)\right|\leq\varepsilon.
\]
Thus we proved the theorem.
\end{proof}
\section*{\large{Universal property of Ridge kernels}}

\section{Universal Approximation Theorem }
 By \cite[Theorem 3.1]{Pinkus2}, we formulate 
 a celebrated   Universal Approximation Theorem  
  \begin{theorem} \label{Universal1}
  If $\sigma\in C(\R)$ is not polynomial  then for any $f\in C(\R^d)$ and any compact set $K\subset \R^d$ and $\varepsilon>0$  there are $\gotm\in \Z_+$,
 $\{w_j\}_{j=1}^\gotm\subset \R^{d}$, $\{a_j\}_{j=1}^\gotm,\{b_j\}_{j=1}^\gotm\subset \R$, such that
\[
\sup_{x\in K}|f(x)-\sum_{j=1}^\gotm a_j \sigma(\lan x,w_j\ran+b_j)|\leq \varepsilon.
\]  
The converse also holds.
\end{theorem}

Define for $\Lambda,B\subset \R$ and $\sigma\in C(\R)$ a set
\[
\calN= \spa \{\sigma(\lambda t -b), \lambda\in \Lambda, b\in B \}
\]
By \cite[Proposition 3.10]{Pinkus2} for a function $\sigma\in C(\R) \cap L^1(\R)$ 
 or  $\sigma$ is continuous non decreasing and bounded (but not the constant function) we get then
 $\calN(\sigma,\{1\},\R)$ is dense in $C(\R)$ in the topology of uniform convergence on compacta.
  The general approach 
  was examined by Schwartz, see \cite[page 160]{Pinkus2},  
where he introduced
the following definition of the class of mean-periodic functions
    \begin{definition}
  A function $f:\R^d \to \R$  is said to be mean-periodic if
  \[
  \spa\{f(\cdot-a):a\in \R^d \}
  \]
  is not dense in $C(\R^d)$ in the topology of uniform convergence on compacta.
  \end{definition}
Luckily we are interested in the univariate case, $d=1$. So we have equivalent formulation.
\begin{lemma}
  A function $\sigma:\R \to \R$  is said to be not mean-periodic if and only if
 $\calN(\sigma,\{1\},\R)$
  is  dense in $C(\R)$ in the topology of uniform convergence on compacta.
  \end{lemma}
  So we can say that a function $\sigma\in C(\R) \cap L^1(\R)$ 
 or  $\sigma$ is continuous non decreasing and bounded (but not the constant function) is not mean-periodic. 
In \cite[Proposition 3.12]{Pinkus1} we find another characterization of not mean-periodic
  functions. 
 By  Vostrecov and Kreines's Theorem \cite[Theorem 3.2]{Pinkus2} we can reformulate main theorem which will be also useful, \cite[Proposition 3.3]{Pinkus2}. An additional assumption on $\sigma$ helps to give more specific domain of vectors $w=\{w_j\}_{j=1}^\gotm$.
 \begin{theorem}\label{meanperiodic}
 Let $\sigma\in C(\R)$ be not mean-periodic. Then for
 any $f\in C(\R^d)$ and any compact set $K\subset \R^d$ and $\varepsilon>0$  there are $\gotm\in \Z_+$,
 $\{w_j\}_{j=1}^\gotm\subset \Sp^{d}$, $\{a_j\}_{j=1}^\gotm,\{b_j\}_{j=1}^\gotm\subset \R$, such that
\begin{equation}
\label{Vost}
\sup_{x\in K}|f(x)-\sum_{j=1}^\gotm a_j \sigma(\lan x,w_j\ran+b_j|\leq \varepsilon.
\end{equation} 
 \end{theorem}
 
This Theorem motivates  following definition.
 \begin{definition}
A function $\sigma \in C(\mathbb{R})$ is said to have the \emph{universal approximation property}
on a set $W \subset \mathbb{R}^d$ if, for any function $f \in C(\mathbb{R}^d)$,
any compact set $K \subset \mathbb{R}^d$, and any $\varepsilon>0$,
there exist $\gotm \in \mathbb{Z}_+$ and parameters
$\{w_j\}_{j=1}^{\gotm} \subset W$,
$\{a_j\}_{j=1}^{\gotm} \subset \mathbb{R}$,
and $\{b_j\}_{j=1}^{\gotm} \subset \mathbb{R}$ such that
\eqref{Vost} holds.
\end{definition}

\section{Universal family}\label{universal}

\begin{lemma}
\label{add}
If $\sigma\in C(\R)$ has  an universal approximation property on a set $W\subset \R^d$, then  a family of functions 
 \[
{ k}(x,y|\vartheta)= \sum_{j=1}^{\gotm} c_j \sigma(\lan x, w_j\ran+b_j) \sigma(\lan y, w_j\ran+b_j), \quad x,y\in \R^d, c_j>0,
\]
is universal for multi-parameters $\vartheta$  which includes both the set of parameters $(\{c_j,b_j,w_j\}_{j=1}^\gotm)$, $c_j>0$, $b_j\in\R$, $w_j\in W$. 
\end{lemma}
 
\begin{proof}
Since $\sigma$ has an universal approximation property on $W$
we find for $f$, $K$ and $\varepsilon>0$  
 multi-parameters $\vartheta$ , i.e. $(\{c_j,b_j,w_j\}_{j=1}^\gotm)$, $c_j>0$, $b_j\in\R$, $w_j\in W$
\[
\sup_{x\in K}|f(x)-\sum_{j=1}^\gotm a_j \sigma(\lan x,w_j\ran+b_j)|\leq \varepsilon.
\]
Taking perhaps smaller $\gotm$ we may assume that 
\begin{equation}
\label{h(x)}
\sum_{j=1}^\gotm a_j \sigma(\lan x,w_j\ran+b_j)=\sum_{j=1}^\gotm a_j \varphi_j(x),\qquad \varphi_j(x)=\sigma(\lan x,w_j\ran+b_j)
\end{equation}
where functions $\varphi_j$, $j=1,\ldots,\gotm$
are  linearly independent over $x\in K$.
Let us define  kernels for $x,y\in K$
\[
{ k}(x,y)= \sum_{j=1}^{\gotm} c_j \sigma(\lan x, w_j\ran+b_j) \sigma(\lan y, w_j\ran+b_j)=
\sum_{j=1}^{\gotm} c_j \varphi_j(y)\varphi_j(x)
\]
with some $c_j>0$, $j=1,\ldots,m$. 
Hence to prove the theorem it is sufficient find coefficients $d_i$ and $x_i\in K$, $i=1,\ldots,\gotm$ and $c_j>0$, $j=1,\ldots,m$
such that
\[
h(x)=\sum_{i=1}^\gotm d_i { k}(x,x_i)=\sum_{i=1}^\gotm d_i \sum_{j=1}^\gotm c_j\varphi_j(x_i)\varphi_j(x)=\sum_{j=1}^\gotm c_j\big(\sum_{l=1}^\gotm d_i \varphi_j(x_i)\big)\varphi_j(x).
\]
Comparing with \eqref{h(x)}  we need to find $d_i\in \R,c_i>0$ and  $x_i\in K$, $i=1,\ldots,\gotm$ such that for all $j=1,\ldots, \gotm$
\begin{equation}
\label{stability}
a_j=c_j \sum_{i=1}^\gotm d_i  \varphi_j(x_i).
\end{equation}
By Lemma \ref{interpolation} there are $x_i\in K$, $i=1,\ldots,\gotm$ such that the matrix $[\sigma(\lan x_i, w_j\ran+\beta_j)]_{i,j=1}^\gotm=[\varphi_j(x_i)]_{i,j=1}^\gotm$ is invertible. With the choice $c_i=|a_i|$ there is a solution of  the linear equations 
\begin{equation}
\label{linear1}
  sign(a_j)= \sum_{i=1}^\gotm d_i  \varphi_j(x_i),
\end{equation}
  where
  \[
  sign(x)=\begin{cases}
  1 & x\geq 0\\
  -1& x<0.
  \end{cases}
  \]

\end{proof} 
 \begin{lemma}\label{interpolation}
  Let $K\subset \R^d$ be compact and let  $\{\varphi_j\}_{j=1}^\gotm$ be a  set of linearly independent continuous function
  $\varphi_j: K\to \R$. 
  Let 
  \[
  V=\spa \{\varphi_j: 1\leq j\leq s\}\subset C(K).
  \]  
   Then $ W:=\{\delta_x:x\in K\}= V^*.$
  \end{lemma} 
  \begin{proof}  
      Define a continuous functional $\delta_x(f)=f(x)$, $x\in K$. Note that 
  \[
  W:=\{\delta_x:x\in K\} \subset V^*.
  \]
  Since for every $f\neq 0$ such that $f\in V$ we find $x\in K$ such that $f(x)\neq 0$ then
  a set $W$ contains a basis in $V^*$.
  \end{proof}
 
\begin{corollary}\label{UAT}
Consider the family of functions depending on parameters $\vartheta=(c,w,b,\gotm)$ such that let
 $\gotm\in {\mathbb N}$ and $w=\{w_j\}_{j=1}^\gotm\subset \R^d$, $c=\{c_j\}_{j=1}^\gotm, b=\{b_j\}_{j=1}^\gotm\subset \R$, with a fixed $\sigma\in C(\R,\R)$
\[
{ k}(x,y|\vartheta)= \sum_{j=1}^{\gotm} c_j \sigma(\lan x, w_j\ran+b_j) \sigma(\lan y, w_j\ran+b_j), \quad x,y\in \R^d, c_j>0,
\]
where $\sigma$ is not a polynomial. Then such family is universal.
\end{corollary} 

\begin{proof} 
By  Theorem \ref{Universal1}   the function
$\sigma$ has a universal approximation property on the set $W=\R^d$. By Lemma \ref{add} we get the Corollary.

\end{proof}

  To obtain a better localization of  vectors $w=\{w_j\}_{j=1}^\gotm$ we need to add some additional assumption on the function $\sigma$. 
 \begin{corollary} Consider the family of functions depending on multi parameters $\vartheta$ that is $(c,w,b,\sigma,\gotm)$ such that let
 $\gotm\in {\mathbb N}$ and $w=\{w_j\}_{j=1}^\gotm\subset \Sp^{d-1}$, $c=\{c_j\}_{j=1}^\gotm, b=\{b_j\}_{j=1}^\gotm\subset \R$,
\[
{ k}(x,y|\vartheta)= \sum_{j=1}^{\gotm} c_j \sigma(\lan x, w_j\ran+b_j) \sigma(\lan y, w_j\ran+b_j), \quad x,y\in \R^d, c_j>0.
\]
Assume that $\sigma\in C(\R)$ is not mean-periodic.
Then such family is universal. 
\end{corollary}
\begin{proof}

 By Theorem \ref{meanperiodic} the function $\sigma$ has 
 universal approximation property on the set $W=\Sp^{d-1}$
  By Lemma \ref{add} we get Corollary.
 \end{proof}

 If we wish to restrict the space of multi-parameters 
$\vartheta$ to a compact set, it is sufficient to assume that the activation function 
$\sigma\in C(\R)$ has compact support. Indeed, if $\sigma\in C(\R)$ is compactly support then it is not mean-periodic. Note that this compact of parameters  depends on a compact set $K\subset \R^d$.
 \begin{proposition}
 Consider the family of functions depending on multi-parameters $\vartheta=(c,w,b,\sigma,\gotm)$ such that
 $\gotm\in {\mathbb N}$ and $w=\{w_j\}_{j=1}^\gotm\subset \Sp^{d-1}$, $c=\{c_j\}_{j=1}^\gotm, b=\{b_j\}_{j=1}^\gotm\subset \R$, a function $\sigma\in C(\R)$ which
 is compactly supported
\[
{ k}(x,y|\vartheta)= \sum_{j=1}^{\gotm} c_j \sigma(\lan x, w_j\ran+b_j) \sigma(\lan y, w_j\ran+b_j), \quad x,y\in \R^d, c_j>0.
\]
 Then there is a compact interval $B\subset \R$ depending on $K$ such that a set of functions
$\{k(x,y|\vartheta)$  is dense in $C(K)$ additional assuming that $b_j\in B$, $j=1,\ldots,\gotm$.
\end{proposition}

\begin{proof}
Still we can find that it is essential to use Universal Approximation Theorem that is  \cite[Theorem 3.1]{Pinkus2}. Let $f\in C(K)$. First we apply  Vostrecov and Kreines's Theorem \cite[Theorem 3.2]{Pinkus2}. That is for given $\varepsilon>0$ there are continuous function $g_j\in C(\R)$ and $w_j\in \Sp^{d-1}$, $j=1,\ldots, n$ such that
\[
\sup_{x\in K} |f(x)-\sum_{j=1}^n g_j(\lan x, w_j \ran)|\leq \varepsilon.
\]
Note that there is $a>0$ depending on $K$ such that for all $w\in \Sp^{d-1}$, $x\in K$ we have $|\lan w,x \ran|\leq a$. Hence we can consider approximation on interval $[-a,a]$. 
Since $\sigma$ has a universal approximation property by Lemma \ref{add} we get Proposition.
\end{proof}

\section{Criteria of choosing Mercer kernels in learning theory }

Now we study how to select the optimal Reproducing Kernel Hilbert Space $\calH_k$ (RKHS) for a learning process, see \cite{Cucker}, \cite{Smale}. Let ${ k}$ be a Mercer kernel corresponding to $\calH_k$, i.e. ${ k}$ is a kernel and continuous. 

Now we present a simplified version of problems of learning theory.
Let $X$ be a compact subset of $\R^d$ and $Z = X \times Y$. Let $\rho$ be a probability measure on $Z$ and $\rho_X, \rho_{Y |x}$  be the induced marginal probability measure on $X$ and conditional probability measure on $\mathbb R$ conditioned on $x\in X$, respectively. Define $f_\rho : X \to {\mathbb R}$ as follows:
\[
f_\rho(x)=\int_Y y d\rho_{y|x}.
\]
This is the regression function of $\rho$. We assume that $f_\rho \in L^2(\rho_X)$. 
As a possible method of finding $f_\rho$ one can minimize the regularized least square problem (variation problem) in $\calH_k$
$ \inf_{f \in \calH_k} \calL(f)$
where
\begin{equation}
\label{functional}
\calL(f)=
\int_Z (f (x) -y)^2 d\rho(x,y) + \lambda \|f\|_{\calH_k}^2,\qquad \lambda >0.
\end{equation}
\cite[Proposition 7 in Chapter III]{Cucker} guarantees the existence and uniqueness of a minimizer.
Usually, the measure $\rho$ is unknown. Now we consider the sampling, let
\[
z = (z_1,\ldots,z_N)=((x_1 , y_1),\ldots , (x_N, y_N))
\]
be a sample in $Z^N$, i.e. $n$ examples independently drawn according to $\rho$.
In the context of RKHS, given a sample $z$, 
discrete version of \eqref{functional} 
\begin{equation}
\label{functional1}
\calL_N(f)=\frac{1}{N}\sum_{i=1}^N(f (x_i ) -y_i )^2 + \lambda \lan f, f \ran_{\calH_k},
\end{equation}
and “batch learning”\ means
solving the regularized least square problem 
\begin{equation}
\label{variation}
f_{\lambda,z} = \mbox{arg\,min}_{f\in \calH_k}\,
\calL_N(f),
\quad
\lambda > 0,
\end{equation}
see \cite[Comparison with “Batch Learning” Results]{Smale}
and \cite{2021}.
The existence and uniqueness of $f_{\lambda,z}$ given as in \cite [Section 6]{Cucker} says
\[
f_{\lambda,z} =
\sum_{i=1}^N
a_i { k}(x,x_i).
\]
where $a = (a_1 ,\ldots, a_N)$ is the unique solution of the well-posed linear system in $\R^n$,
where 
\[
a=(G_N+\lambda N I)^{-1}y,\qquad y=(y_1,\ldots,y_N),\qquad G_N=[{ k}(x_i,x_j)]_{1\leq i,j\leq N}.
\]
\cite [Section 6]{Cucker}.

\begin{lemma} 
\label{61}
 Let $z \in  Z^N$. For $\lambda>0$
\[
\calL_N(f_{\lambda,z})= \lambda y'  \left(G_N+ \lambda N I\right)^{-1}y
\]
\end{lemma}
\begin{proof}
Now let put $f_{\lambda,z}$ and calculate
\[
\calL_N(f_{\lambda,z})=\frac{1}{N}\sum_{j=1}^N(f_{\lambda,z} (x_j ) -y_j )^2 + \lambda \lan f_{\lambda,z}, f_{\lambda,z} \ran_{\calH_k},
\]
Note that
\[
\frac{1}{N}\sum_{i=1}^N(f_{\lambda,z} (x_i ) -y_i )^2= \frac{1}{N}\|G_N(G_N+\lambda N I)^{-1} y-y\|^2.
\]
But
\[
\begin{aligned}
G_N(G_N+\lambda N I)^{-1}-I&=G_N(G_N+\lambda N I)^{-1}-(G_N+\lambda N I)(G_N+\lambda N I)^{-1}\\
&=-\lambda N I(G_N+\lambda N I)^{-1}.
\end{aligned}
\]
From definition of ${ k}$ and $a$
\[
\lan f_{\lambda,z}, f_{\lambda,z} \ran_{\calH_k}=a'G_N a=y'(G_N+\lambda N I)^{-1} G_N (G_N+\lambda N I)^{-1}y.
\]
Hence
\[
\calL_N(f_{\lambda,z})=N\lambda^2\|(G_N+\lambda N I)^{-1} y \|^2 + \lambda y'(G_N+\lambda N I)^{-1} G_N (G_N+\lambda N I)^{-1}y.
\]
Consequently,
\[
\begin{aligned}
\calL_N(f_{\lambda,z})&=N\lambda^2 y' (G_N+\lambda N I)^{-1} (G_N+\lambda N I)^{-1} y  + \lambda y'(G_N+\lambda N I)^{-1} G_N (G_N+\lambda N I)^{-1}y.
\\
&=\lambda\big( y' (G_N+\lambda N I)^{-1} (\lambda N I  +  G_N) (G_N+\lambda N I)^{-1}y\big)
\\
&=\lambda\big( y'  (G_N+\lambda N I)^{-1}y\big)
\end{aligned}
\]
\end{proof}
%\section{Numerical simulation and conclusions} 
The Von Neumann series gives that
\[
\left(\frac{1}{\lambda N}G_N+ I\right)^{-1}=\sum_{j=0}^\infty (-1)^j \left( \frac{1}{\lambda N}\right)^j G^j_N
\]
The series converge in norm topology if ${\lambda}N > \|G_N\|$.  So for any $\epsilon>0$ there is $L>0$ such that
\[
\left\|\left(\frac{1}{N \lambda}G_N+ I\right)^{-1}-\sum_{j=0}^L (-1)^j \left( \frac{1}{\lambda N}\right)^j G^j_N\right\|<\epsilon
\]
Now by Spectral Theorem, we can decompose  matrix $G_N$ as
\[
G_N=U' \Lambda  U,
\]
where $U$ is an orthonormal  matrix consisting of eigenvectors of of $G_N$ and $\Lambda$ is a diagonal matrix
consisting of eigenvalues of of $G_N$. Hence $G^j_{N}=U' \Lambda^j  U$.

Now we want chose best RKHS $\calH$ defined by 
a  ridge kernel. Let $\vartheta=(c,w,b,\sigma)$ denote the parameters,
such that for any  $\gotm\in {\mathbb N}$ we take parameters  $c_j>0, w_j\in \Sp^{d-1}, b_j\in \R$, $j=1,\ldots,\gotm$ and a function $\sigma\in C(\R,\R)$ not mean periodic 
\begin{equation}
\label{propose}
{ k}(x,y|\vartheta)= \sum_{j=1}^{\gotm} c_j \sigma(\lan x, w_j\ran+b_j) \sigma(\lan y, w_j\ran+b_j), \quad x,y\in \R^n.
\end{equation}
In order to emphasize  the connection kernels $k$ with parameters $\vartheta$,  we use the notion:
\[
G_N= G_{N,\vartheta}=[{ k}(x_i,x_j|\vartheta)]_{1\leq i,j\leq N}
 \]
Our first proposition on estimation of the best parameters is the following. Given a number $m$, a sample $z$, and $\lambda>0$
we find parameters $\vartheta$ such that
\begin{equation}
\label{optymal}
\min_{\vartheta} \calL_N(f_{\lambda,z})=\min_{\vartheta}
\lambda\big( y'  (G_{N,\vartheta}+\lambda N I)^{-1}y\big).
\end{equation}

Second approach is possible for given  $L$. Let $\lambda_{k,\vartheta}, k=1,\ldots,N$ be a sequence of eigenvalues of 
$G_{N,\theta}$, $\lambda_{k,\theta}\geq  \lambda_{k+1,\theta}$. Let ${\lambda}N> \lambda_{1,\theta}$. A criterion 
for optimal $\vartheta$
\begin{equation}
\label{Newman}
\min_{\vartheta} \frac{1}{N} y'U_\theta' \left( \sum_{j=0}^L (-1)^j \left( \frac{1}{\lambda N}\right)^j \Lambda^j_{n,\theta} \right) U_\vartheta y.
\end{equation}
Usually it is assumed in  that this probability measure $\rho_{y|x}$ is supported on an interval, see  \cite{Binev}[In this paper, it is assumed
that this probability measure is supported on an interval].   We can assume that $|y|\leq 1$.

 In next section we propose a modified version of our algorithm.

 \section{ Interpolation of a shallow neural network meets the learning theory}

 In the proof of Theorem \ref{UAT} for a given compact set $K$ and $\varepsilon>0$  we find the $w_j\in \Sp^{d-1},\beta_j\in \R$, $j=1,\ldots, \gotm$ such that 
  functions $\varphi_j(x)=\sigma(\lan x,w_j\ran+\beta_j)$ are linearly independent over $K$, next we look for  $x_i\in K$ such that $\det[\sigma(\lan x_i, w_j\ran+\beta_j)]_{i,j=1}^m\neq 0$.
  In \cite[Theorem 5.1]{Pinkus2} the problem  is inverse. We have $x_i\in K$, $i=1,\ldots,n$  and we look for coefficients $w_i\in \Sp^{d-1}$ and $b_i\in \R$ $i=1,\ldots,n$
  such that $ \det [\varphi_j(x_i)]_{i,j=1}^n\neq 0$. This is a classical interpolation problem. 
 \begin{theorem}
 \label{interp} Let $\sigma \in C(\mathbb{R})$ be a function which is not a polynomial.
For any $n$ distinct points $\{x_i\}_{i=1}^n \subset \mathbb{R}^d$
there exist parameters
$\{w_j\}_{j=1}^n \subset \mathbb{R}^d$ and
$\{\beta_j\}_{j=1}^n \subset \mathbb{R}$
such that
\[
\det\!\big[\sigma(\langle x_i,w_j\rangle+\beta_j)\big]_{1\le i,j\le n}\neq 0 .
\]

In particular, the interpolation problem is solvable: for any
$\{\alpha_i\}_{i=1}^n \subset \mathbb{R}$ there exist coefficients
$\{d_j\}_{j=1}^n \subset \mathbb{R}$ such that
\[
\sum_{j=1}^n d_j \sigma(\langle x_i,w_j\rangle+\beta_j)=\alpha_i,
\quad i=1,\ldots,n .
\]

Furthermore, if $\sigma$ is not mean periodic, then the parameters
$\{w_j\}_{j=1}^n$ may be chosen in the sphere $\mathbb{S}^{d-1}$.
\end{theorem}

 If $\lambda>0$ is small, functional \eqref{functional1} is realized by an interpolation problem.
\begin{proposition}\label{propozycja}
Let $\sigma \in C(\mathbb{R})$ be a function which is not a polynomial.
Let
\[
z = (z_1,\ldots,z_n)=((x_1 , y_1),\ldots , (x_n, y_n))
\]
be a sample. Then there exist parameters $\{w_j\}_{j=1}^n \subset \mathbb{R}^d$,
$c_j>0$, and $\beta_j\in \mathbb{R}$, $j=1,\ldots,n$, such that for the kernel
\[
k(x,y)=\sum_{j=1}^n c_j \sigma(\langle x,w_j\rangle+\beta_j)
\sigma(\langle y,w_j\rangle+\beta_j),
\]
there exist coefficients $\{a_i\}_{i=1}^n$ for which the function
\[
g(x)=\sum_{i=1}^n a_i\, k(x,x_i)
\]
interpolates the data, that is,
\[
g(x_i)=y_i, \quad i=1,\ldots,n.
\]

Furthermore, if $\sigma$ is not mean periodic, then the parameters
$\{w_j\}_{j=1}^n$ may be chosen in the sphere $\mathbb{S}^{d-1}$.
\end{proposition}
\begin{proof}
By Theorem~\ref{interp}, for the sample $(x_i,y_i)_{i=1}^n$ there exists
$\{w_j\}_{j=1}^n\subset \R^d$  and
$\{\beta_j\}_{j=1}^n\subset \R$  such that
\begin{equation}
\label{deter}
\det\big[\sigma(\langle x_i,w_j\rangle+\beta_j)\big]_{1\le i,j\le n}\neq 0 .
\end{equation}
and a function
\[
f(x)=\sum_{j=1}^n d_j \sigma(\langle x,w_j\rangle+\beta_j)
\]
that interpolates the data, i.e.
\[
f(x_i)=y_i, \quad i=1,\ldots,n .
\]

Consider now the kernel
\[
k(x,y)=\sum_{j=1}^n c_j 
\sigma(\langle x,w_j\rangle+\beta_j)
\sigma(\langle y,w_j\rangle+\beta_j),
\]
where $c_j>0$ will be taken later.
We seek coefficients $\{a_i\}_{i=1}^n$ such that
\begin{equation}\label{prop}
\begin{aligned}
f(x)
&=\sum_{i=1}^n a_i k(x,x_i) \\
&=\sum_{j=1}^n c_j \sigma(\langle x,w_j\rangle+\beta_j)
\sum_{i=1}^n a_i \sigma(\langle x_i,w_j\rangle+\beta_j).
\end{aligned}
\end{equation}

Comparing the coefficients in the representation above, we obtain the system
\begin{equation}\label{solution}
d_j
=
c_j \sum_{i=1}^n a_i \sigma(\langle x_i,w_j\rangle+\beta_j),
\quad j=1,\ldots,n .
\end{equation}

By \eqref{deter}
 the system \eqref{solution} has a solution. We can take $c_j$  choosing  $c_j=|d_j|\neq 0$, $j=1,\ldots,n$ or $c_j=1$ for all $j$. Thus we obtain coefficients $a_i$, $i=1,\ldots,n$, satisfying \eqref{solution}, which completes the proof.
\end{proof}
 If a number of observed  $x_i\in K$, $i=1,\ldots,n$, is large we should select  a smaller set multi-parameters depending on $\gotm<n$ that is $\{w_j\}_{j=1}^\gotm\subset \Sp^d$  and
$\{\beta_j\}_{j=1}^\gotm\subset \R$. Even then, the optimization problem \eqref{optymal} still involves a large number of multi-parameters. 
In general there is no good method for
a regression function by fully connected  neural networks. 
 In \cite{Bagirov} they examine $\sigma=ReLu$ and they write  {\it
to compute the estimate for given numbers of linear functions, we have to minimize
\[
\frac{1}{n}\sum_{i=1}^n \left| \left( \max_{k=1,\ldots K}\min_{j=1,\ldots L_k} (\lan w_{k,j},x_i \ran -b_{k,j})\right) -y_i\right|^2.
\]
Unfortunately, we cannot solve this minimization problem exactly in general. The reason is that the function to be minimized
is nonsmooth and nonconvex.}
They obtain finely the optimal rate of estimation which was proved by Stone 1982 but still the cost is huge.  

Our idea is the following. Machine learning \cite{Cucker}, \cite{Smale}, \cite{2021} gives precise the estimation of error of the defect function for fixed Mercer's kernel. We want to a change kernels to obtain more suitable estimation of the regression function itself. Somehow the paper of \cite{Belkin} is encouraging, {\it Fit without fear: remarkable mathematical
phenomena of deep learning through the prism of
interpolation}. So our task is to find numerically the equilibrium between this two approach. 
We propose an adaptive procedure which can help in this situation. 

\subsection{Algorithm}\label{sec:algorithm} 
Let the sample consist of $n$ data points. We divide the sample into $L$ disjoint sub-samples,
each containing $N_l$ data points, $l=1,\ldots,L$, such that
\[
N_1+\cdots+N_L=n.
\]

In the first step, we choose a ridge kernel that interpolates the first sub-sample of size $N_1$
according to Proposition~\ref{propozycja}. We set $\gotm_1=N_1$ and define
\[
k_1(x,y)
=
\sum_{j=1}^{\gotm_1}
c_j \,
\sigma(\langle x,w_j\rangle+\beta_j)\,
\sigma(\langle y,w_j\rangle+\beta_j).
\]

In each subsequent step, we solve the optimization problem~\eqref{optymal}
using a new batch of data of size $N_l$, $l=2,\ldots,L$. Simultaneously, we refine the kernel
by adding new multi-parameters, whose number depends on $N_l$.

Assume that after $l\geq 1$ steps we have constructed a kernel
\[
k_l(x,y)
=
\sum_{j=1}^{\gotm_l}
c_j \,
\sigma(\langle x,w_j\rangle+\beta_j)\,
\sigma(\langle y,w_j\rangle+\beta_j).
\]

At step $l+1$, we define an updated kernel with $\gotm_{l+1}$ multi-parameters,
introducing $\Delta\gotm=\gotm_{l+1}-\gotm_l$ new ones:
\[
k_{l+1}(x,y)
=
k_l(x,y)
+
\sum_{j=\gotm_l+1}^{\gotm_{l+1}}
c_j \,
\sigma(\langle x,w_j\rangle+\beta_j)\,
\sigma(\langle y,w_j\rangle+\beta_j).
\]
\section{Experimental results}

\subsection{Universal family}

Iris is a labeled dataset of $n = 150$ flowers of $k=3$ species, namely \textit{setosa}, \textit{virginica} and \textit{versicolor} (50 each). For each observation we have 4 measurements of length and width, so original feature space  is $X=\R^4$ and $Y= \lbrace 1, 2, 3 \rbrace$. The sample space is $\mathcal{X} =X^n$ and $\mathcal{Y} = Y^n$. As a preprocessing step we shifted and scaled each feature to have zero mean and unit variance across whole sample.

In order to classify single observation into one of three available labels we use one-vs-all strategy, i.e. for each $k=1,2,3$ we create new labels $y_i^* = \mathbf{1}_{\lbrace y_i = k \rbrace}$ and then fit binary classifier $f_k$. Then $i$-th observation is classified by evaluating predictions from all $k$ classifiers and assigning the label corresponding to the classifier with the highest $f_k(x_i)$ value.

In our experiment we choose the following parameters:

\begin{enumerate}
	\item $m = 2$ - number of terms in Mercer kernel approximation, which gives us $m(2+d) = 12$ parameters $\theta = (c, b, w)$ to be estimated
	\item $\lambda = 0.01$ - regularization parameter.
\end{enumerate}

Due to relatively small dataset it was possible to apply Lemma \ref{61} by inverting relevant matrix and minimizing loss function directly. Additionally, we approximated inverse matrix with Neumann series by minimizing \eqref{Newman}. We performed numerical optimization using L-BFGS-B algorithm.

During the experiment we encountered problems with numerical stability of two kinds:

\begin{enumerate}
	\item signularities when inverting matrix $K_{\lambda} = G_N + \lambda n I$, which appears in Lemma \ref{61} and during solving linear system $a = K_{\lambda}^{-1} y$
	\item infinity in Neuman approximation of inverse matrix.
\end{enumerate}

We overcame the first problem by QR decomposition of $K_{\lambda}$ matrix. For Neumann approximation we used $L=5$. Here we observed high variance of accuracy when conducting this experiment repeatedly, suggesting that L-BFGS-B optimization with Neumann approximation is sensitive to random initialization of parameters $\theta$. In all simulations we drew initial parameters independently from the uniform and normal distributions:

$$
\begin{cases}
c_j \sim \mathcal{U}(0,1) \\
b_j \sim \mathcal{N}(0,1) \\
w_{ij} \sim \mathcal{N}(0,1)
\end{cases}
$$

where $i=1,2,3,4$ is dimension of input space $X$ and $j = 1,2$ is number of terms in Mercer kernel approximation of form 1.1.

\begin{figure}[h!]
   \includegraphics[totalheight=5.5cm]{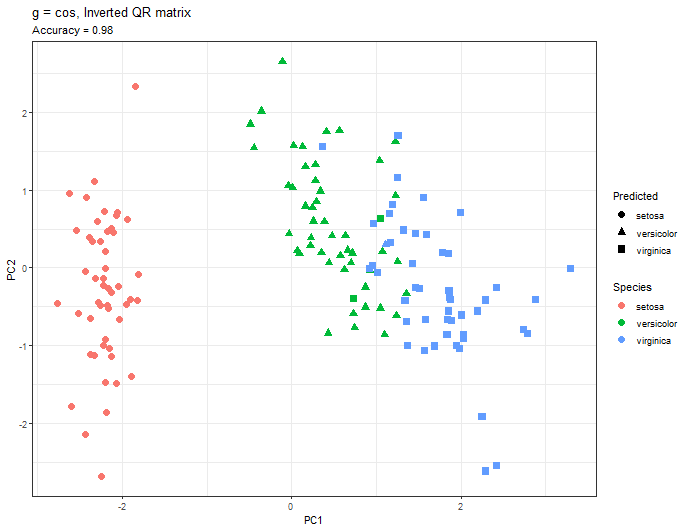}
\end{figure}

\begin{figure}
   \includegraphics[totalheight=5.5cm]{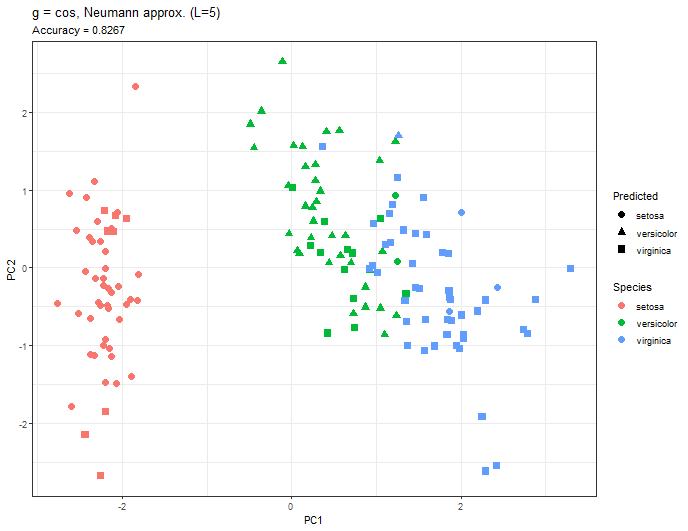}
\end{figure}

\begin{figure}
   \includegraphics[totalheight=5.5cm]{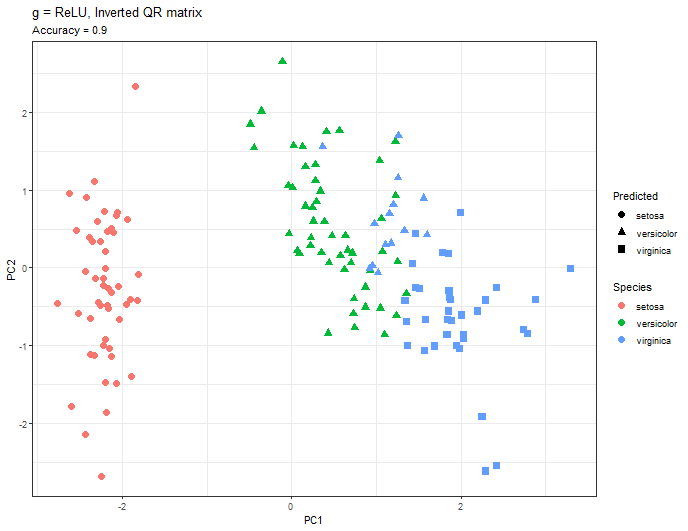}
\end{figure}

Numerical results are presented in Table 1. True labels and predicted labels are presented on Figures 1-3. Shifted and scaled dataset is projected onto spaced spanned by first two principal components. $PC_1$ and $PC_2$ together explain $95.81\%$ of dataset's variance.

\begin{table}[h!]
\begin{tabular}{lll}
\hline
\textbf{Method}         & \textbf{Function} & \textbf{Accuracy}   \\ \hline
Inverted QR matrix      & ReLU              & $0.9$               \\ \hline
Inverted QR matrix      & $\cos$             & $0.98$              \\ \hline
Neumann approx. ($L=5$) & ReLU              & optimization failed \\ \hline
Neumann approx. ($L=5$) & $\cos$             & $0.8267$            \\ \hline
\end{tabular}
\end{table}

\newpage

In this subsection we validate the adaptive algorithm of Section~\ref{sec:algorithm} on a one-dimensional regression problem.  All experiments use the ridge kernel~\eqref{kRR} with the loss functional from Lemma~\ref{61}, solved via QR decomposition. The optimisation problem~\eqref{optymal} is handled by the L-BFGS-B algorithm.

\subsection{Experimental setting}\label{sec:exp_setting}
Let $X = [0,1]$ with marginal distribution $\rho_X = \mathrm{Uniform}(0,1)$. The conditional distribution is $\rho(y\mid x) = f_\rho(x) + \ve$, where $\ve \sim \calN(0,\sigma_\ve)$ with $\sigma_\ve = 0.1$, and the regression function is the third Chebyshev polynomial
mapped to $[0,1]$:
\[
f_\rho(x) = T_3(2x - 1) = 4(2x-1)^3 - 3(2x-1).
\]
A sample of size $n = 100$ is drawn.

\begin{figure}[H]
   \includegraphics[totalheight=5.5cm]{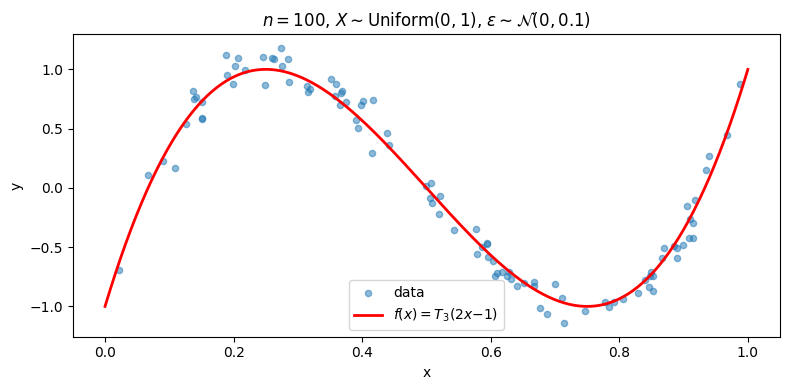}
\end{figure}
 We apply the adaptive procedure of Section~\ref{sec:algorithm} with the parameters listed in Table~\ref{tab:exp_params}.

\begin{table}[h!]
\centering
\caption{Experimental parameters.}
\label{tab:exp_params}
\begin{tabular}{lll}
\hline
\textbf{Symbol} & \textbf{Value} & \textbf{Description} \\ \hline
$n$ & $100$ & Total sample size \\
$N_1$ & $10$ & First sub-sample size (interpolated via Proposition~\ref{propozycja}) \\
$N_l,\; l \ge 2$ & $10$ & Subsequent batch size \\
$\Delta \gotm$ & $2$ & New kernel terms added per step \\
$\lambda$ & $0.01$ & Regularisation parameter \\
$\sigma$ & $\cos$ & Activation function \\ \hline
\end{tabular}
\end{table}

Initial multi-parameters are drawn independently as
$c_j \sim \mathcal{U}(0,1)$, $\beta_j \sim \calN(0,1)$,
$w_{ij} \sim \calN(0,1)$. 

\subsection{Adaptive algorithm}
The algorithm proceeds in $L = 10$ steps.

Step 1 (Interpolation). The first $N_1 = 10$ data points are interpolated exactly by constructing a ridge kernel with $\gotm_1 = N_1 = 10$ terms according to Proposition~\ref{propozycja}. We find $\{w_j\}_{j=1}^{10} \subset \R$, $\{c_j\}_{j=1}^{10}$, $\{\beta_j\}_{j=1}^{10} \subset \R$ such that the kernel
\[
k_1(x, y) = \sum_{j=1}^{\gotm_1} c_j \, \sigma(\lan x, w_j \ran + \beta_j)\,
\sigma(\lan y, w_j \ran + \beta_j)
\]
interpolates the first sub-sample.  In practice, random weight matrices are generated until the Gram matrix $G_{N_1} = [k_1(x_i, x_j)]_{1 \le i,j \le N_1}$ has condition number below $10^{12}$.

Steps $l = 2, \ldots, 10$ (Optimisation). At each step $l$, a new batch of $N_l = 10$ data points is incorporated and $\Delta \gotm = 2$ new kernel multi-parameters are introduced.  Only the new block of parameters is optimised by minimising the loss from Lemma~\ref{61}:
\[
\calL_N(f_{\lambda,z}) = \lambda\, \mathbf{y}^\top
(G_{N,\vartheta} + \lambda N I)^{-1}\, \mathbf{y},
\]
over the new block, keeping all previously fitted blocks fixed.  After all 10 steps the adaptive kernel contains $\gotm = 10 + 9 \times 2 = 28$ terms and uses all $n = 100$ data points.

Table~\ref{tab:adaptive_steps} reports the loss $\calL_N$ at each step. The normalised quantity $\calL_N / N$ is the proper per-point comparison and is expected to decrease (or stabilise) as the kernel improves.

\begin{table}[h!]
\centering
\caption{Adaptive steps for $\sigma = \cos$, $\lambda = 0.01$.}
\label{tab:adaptive_steps}
\begin{tabular}{rrrrll}
\hline
\textbf{Step} & $N$ & $\Delta \gotm$ & $\gotm$ & $\calL_N$ & $\calL_N / N$ \\ \hline
1 & 10 & 10 & 10 & 0.001102 & 0.00011025 \\
2 & 20 & 2 & 12 & 0.004077 & 0.00020387 \\
3 & 30 & 2 & 14 & 0.007997 & 0.00026656 \\
4 & 40 & 2 & 16 & 0.008885 & 0.00022213 \\
5 & 50 & 2 & 18 & 0.008757 & 0.00017515 \\
6 & 60 & 2 & 20 & 0.009078 & 0.00015130 \\
7 & 70 & 2 & 22 & 0.009918 & 0.00014168 \\
8 & 80 & 2 & 24 & 0.009885 & 0.00012356 \\
9 & 90 & 2 & 26 & 0.009250 & 0.00010277 \\
10 & 100 & 2 & 28 & 0.009747 & 0.00009747 \\ \hline
\end{tabular}
\end{table}

The fitted function $\hat{f}_l$ at each of the 10 adaptive steps shows the following behaviour. After Step~1 the kernel interpolates only 10 points, so the fit outside the
first sub-sample is poor.  As subsequent batches and kernel terms are added, the estimator progressively captures the cubic shape of $f_\rho$.  By Step~5 ($n = 50$, $\gotm = 18$) the fit is visually close to the true regression function. Remaining steps refine the tails and reduce residual oscillation.

\subsection{Error convergence} To measure approximation quality we evaluate on a dense grid of 500 equally-spaced points in $[0,1]$ without noise.  This yields the test RMSE $\|f_\rho - \hat{f}\|$, which measures recovery of the true regression function rather than fit to the noisy observations.  Three diagnostics are tracked across adaptive steps:

\begin{enumerate}
\item \textbf{Test RMSE} --- $\bigl(\tfrac{1}{500}\sum_{i=1}^{500}
   (f_\rho(x_i) - \hat{f}(x_i))^2\bigr)^{1/2}$ on the noise-free grid.
\item \textbf{Train RMSE} --- root mean squared residual on the noisy training data.
\item \textbf{Normalised loss} $\calL_N / N$ --- the per-point value of the
   objective~\eqref{optymal}, comparable across steps with different~$N$.
\end{enumerate}

\begin{table}[h!]
\centering
\caption{Error diagnostics at each adaptive step ($\sigma = \cos$, $\lambda = 0.01$).}
\label{tab:error_diag}
\begin{tabular}{rlll}
\hline
\textbf{Step} & \textbf{Test RMSE} & \textbf{Train RMSE} & $\calL_N / N$ \\ \hline
1 & 0.6136 & 0.0311 & 0.00011025 \\
2 & 0.5853 & 0.0635 & 0.00020387 \\
3 & 0.2907 & 0.0889 & 0.00026656 \\
4 & 0.0390 & 0.0936 & 0.00022213 \\
5 & 0.0610 & 0.0927 & 0.00017515 \\
6 & 0.0556 & 0.0945 & 0.00015130 \\
7 & 0.0248 & 0.0987 & 0.00014168 \\
8 & 0.0279 & 0.0986 & 0.00012356 \\
9 & 0.0295 & 0.0955 & 0.00010277 \\
10 & 0.0329 & 0.0978 & 0.00009747 \\ \hline
\end{tabular}
\end{table}

\begin{figure}[H]
   \includegraphics[totalheight=5.5cm]{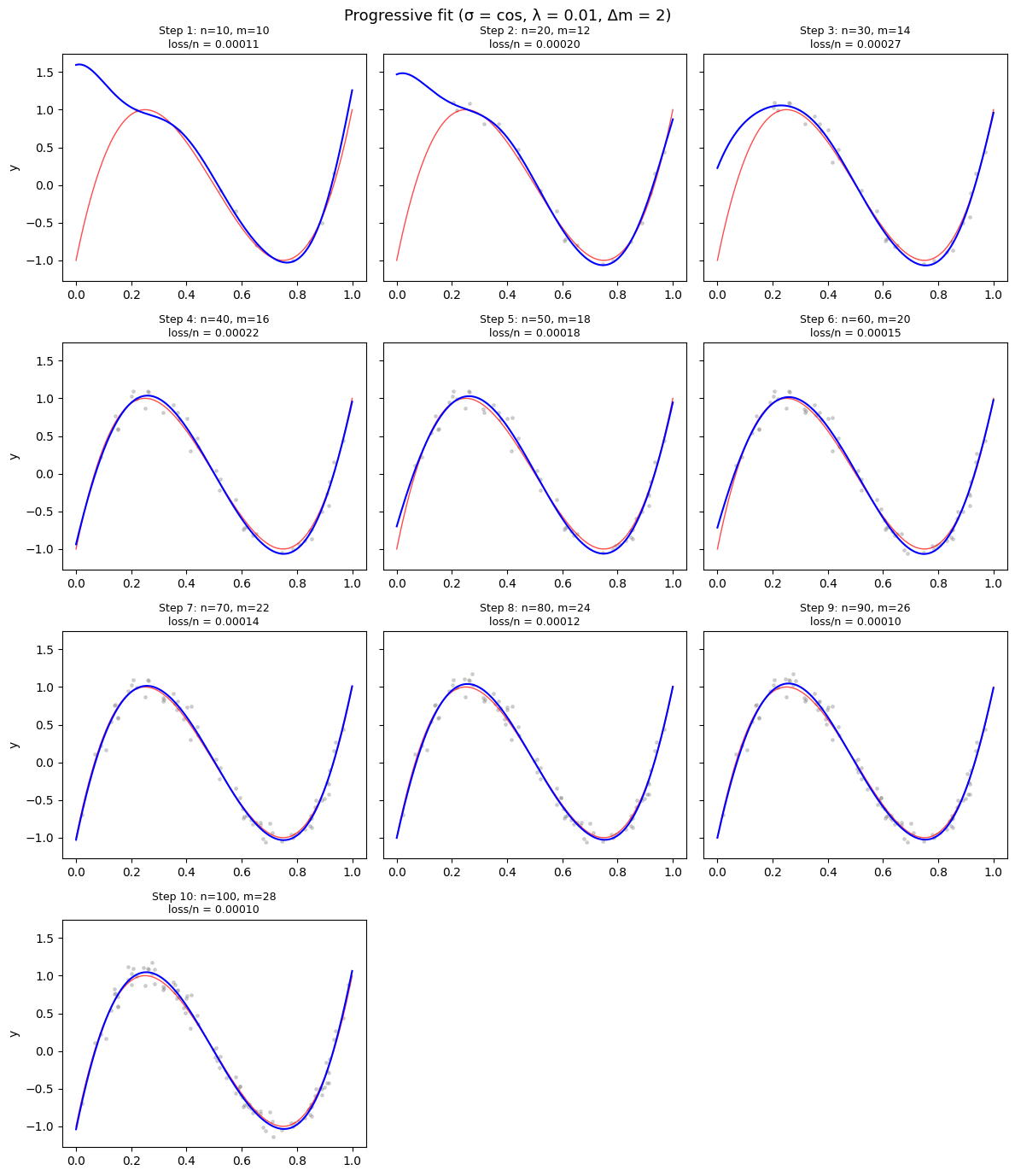}
\end{figure}

The test RMSE drops sharply from $0.61$ at Step~1 to $0.04$ by Step~4 ($n = 40$, $\gotm = 16$) and stabilises
below $0.04$ thereafter, confirming that the adaptive procedure rapidly recovers~$f_\rho$. The train RMSE remains around $0.09$--$0.10$, which is consistent with the noise level $\sigma_\ve = 0.1$. The fact that the test RMSE is substantially lower than the train RMSE reflects that the estimator smooths out the noise and converges
to the true function rather than overfitting the observations.

\subsection{Comparison of activation functions} To investigate the role of the activation function $\sigma$ in the ridge kernel~\eqref{kRR}, we repeat the adaptive experiment with three choices:

\begin{enumerate}
\item $\sigma = \cos$ --- the cosine activation, corresponding to the Rahimi--Recht
   kernel~\cite{Rahimi};
\item $\sigma = \mathrm{ReLU}$ --- the rectified linear unit;
\item $\sigma = \tfrac{1}{2}\cos + \tfrac{1}{2}\,\mathrm{ReLU}$ --- a convex
   combination (homotopy) of the two.
\end{enumerate}

All other parameters ($n$, $N_l$, $\Delta\gotm$, $\lambda$, seed) are held fixed.  By Corollary~\ref{UAT}, each of these activations is not a polynomial, so the corresponding family of ridge kernels is universal in the sense of Definition~\ref{universalfamily}.

\begin{table}[h!]
\centering
\caption{Final test RMSE after 10 adaptive steps for each activation.}
\label{tab:activation_comparison}
\begin{tabular}{lr}
\hline
\textbf{Activation} $\sigma$ & \textbf{Test RMSE} \\ \hline
$\cos$ & 0.0329 \\
$\mathrm{ReLU}$ & 0.1098 \\
$\tfrac{1}{2}\cos + \tfrac{1}{2}\,\mathrm{ReLU}$ & 0.0369 \\ \hline
\end{tabular}
\end{table}

The cosine activation reaches the lowest final test RMSE of $0.033$, closely followed by the homotopy at $0.037$.  The ReLU activation converges more slowly and plateaus at $0.110$.

\subsection{Discussion}

The experiments confirm several theoretical predictions from the preceding sections. The interpolation step reliably produces a well-conditioned Gram matrix for $\gotm_1 = N_1 = 10$ when $\sigma = \cos$ (non-polynomial), validating the constructive proof of Proposition~\ref{propozycja}.

Adding data and kernel terms while keeping previous blocks fixed yields a sharp decrease in test RMSE during the few initial steps. This indicates that the adaptive scheme of Section~\ref{sec:algorithm} quickly reduces the gap between the RKHS estimator and the true regression.

These results demonstrate that the adaptive algorithm of Section~\ref{sec:algorithm}, which bridges interpolation-based kernel construction (Proposition~\ref{propozycja}) with the classical regularised risk minimisation framework of~\eqref{functional}--\eqref{variation}, provides an effective and theoretically grounded approach to nonparametric regression with ridge kernels.

\bibliographystyle{amsplain}

\end{document}